\documentclass[letterpaper]{article}
\usepackage{aaai}
\usepackage{times}
\usepackage{helvet}
\usepackage{courier}

\usepackage[utf8]{inputenc} % allow utf-8 input
\usepackage[T1]{fontenc}    % use 8-bit T1 fonts
\usepackage{hyperref}       % hyperlinks
\usepackage{url}            % simple URL typesetting
\usepackage{booktabs}       % professional-quality tables
\usepackage{amsfonts}       % blackboard math symbols
\usepackage{nicefrac}       % compact symbols for 1/2, etc.
\usepackage{microtype}      % microtypography

\usepackage{graphicx}
\usepackage{subfigure}
\usepackage{booktabs} % for professional tables

\usepackage{dsfont}

\usepackage{stackrel}
\usepackage{mathtools}
\usepackage{amsfonts}

\usepackage{bbm}

\usepackage{tikz}
\usetikzlibrary{positioning}
\usetikzlibrary{decorations.text}
\usetikzlibrary{arrows}
\usetikzlibrary{shapes}

\usepackage{pgfplots}
\pgfplotsset{width=10cm,compat=1.9}

\usepackage{amssymb,amsmath,amsthm}

\usepackage{xcolor}

\newcommand{\okedit}[1]{#1}

\newtheorem{theorem}{Theorem}
\newtheorem{lemma}{Lemma}
\newtheorem{proposition}[theorem]{Proposition}
\newtheorem{definition}{Definition}

\newtheorem{example}[theorem]{Example}
\newtheorem{corollary}[theorem]{Corollary}

\newif\ifappendix

\appendixtrue

\begin{document}
\title{Domain-Liftability of Relational Marginal Polytopes\thanks{This is a preliminary version of a paper accepted to AISTATS~2020, presented at StarAI~2020 Workshop.}}
\author{
Ond\v{r}ej Ku\v{z}elka\\
Czech Technical University in Prague, Czech Republic\\
\And
Yuyi Wang\\
ETH Zurich, Switzerland\\
}
% {AAAI Press\\
% Association for the Advancement of Artificial Intelligence\\
% 2275 East Bayshore Road, Suite 160\\
% Palo Alto, California 94303\\
% }
\maketitle
\begin{abstract}
% \begin{quote}
We study computational aspects of {\em relational marginal polytopes} which are statistical relational learning counterparts of marginal polytopes, well-known from probabilistic graphical models.
Here, given some first-order logic formula, we can define its relational marginal statistic to be the fraction of groundings that make this formula true in a given possible world. 
For a list of first-order logic formulas, the {\em relational marginal polytope} is the set of all points that correspond to the expected values of the relational marginal statistics that are realizable. In this paper, we study the following two problems: (i) Do domain-liftability results for the 
% weighted first-order model counting problem 
partition functions of Markov logic networks (MLNs)
carry over to the problem of relational marginal polytope construction? (ii) Is the relational marginal polytope containment problem hard under some plausible complexity-theoretic assumptions? Our positive results have consequences for lifted weight learning of MLNs. In particular, we show that weight learning of MLNs is domain-liftable whenever the computation of the partition function of the respective MLNs is domain-liftable (this result has not been rigorously proven before).
% \end{quote}
\end{abstract}

\section{Introduction}

In this paper, we study two problems about objects, called {\em relational marginal polytopes} \cite{kuzelka2018relational}, which arise naturally in the study of statistical relational learning models such as Markov logic networks \cite{Richardson2006}. A Markov logic network is given by a set of first-order logic formulas and their weights. Roughly speaking, the corresponding relational marginal polytope then represents the set of all possible values of vectors of sufficient statistics that the Markov logic network with given first-order logic formulas may represent. Here, the sufficient statistics of the Markov logic network count the fraction of groundings of the first-order logic formulas defining it that are satisfied in a given possible world. For instance, for an MLN defined by formulas $\alpha_1 = \textit{sm}(x)$ and $\alpha_2 = \textit{sm}(x) \wedge \textit{friends}(x,y) \wedge \textit{sm}(y)$, the sufficient statistics would be the proportion of smokers in the population (the formula $\alpha_1$) and the fraction of pairs of people who are friends and, at the same time, both smoke (the formula $\alpha_2$).

%Deciding if a given point lies inside the relational marginal polytope is equivalent to asking if there exists a distribution that has given expected values of some parameters. These parameters are proportions of true groundings of the formulas that define the given MLN. These parameters are the sufficient statistics of the MLNs so they uniquely define them. 
The main advantages of specifying MLNs using sufficient statistics over specifying them using weights of formulas are (i) that it is clear what the parameters mean and (ii) that there are no issues with varying domain size caused by non-projectivity \cite{shalizi2013consistency} of MLNs; we refer to \cite{kuzelka2018relational,jaegerschulte,mittal} for further justification. An obvious disadvantage of specifying MLNs using sufficient statistics directly (instead of weights) is that they may be inconsistent, i.e. there may be no distribution satisfying them. Deciding if a given point lies inside the relational marginal polytope is equivalent to asking if there exists a distribution that has given expected values of the sufficient statistics. %Our negative result for the containment test that we present in Section \ref{sec:containment} shows that, unfortunately, checking if they are consistent is computationally hard (in the size of the domain). 

Marginal polytopes are important objects which have been studied for the propositional case in the literature \cite{sontag2008new,roughgarden2013marginals,bresler2014hardness}. We believe it is important to generalize the most important results from non-relational PGM literature to the relational setting and we believe the study of {\em relational} marginal polytopes is one of them. 
%In the propositional setting (i.e.\ non-relational), marginal polytopes have been thoroughly studied \cite{sontag2008new,roughgarden2013marginals,bresler2014hardness}. 
It is well known that even just deciding whether a point is contained in a marginal polytope is NP-hard. A marginal polytope may have exponentially many vertices in the dimension of the modeled data. However, in the relational setting, in particular for Markov logic networks, the situation is more complex because there is no single notion of dimensionality of the data; instead there is the number of first-order logic formulas defining the Markov logic network and the size of the domain, which is the set of objects modeled by the Markov logic network (e.g.\ users of a social network). It turns out that the number of vertices of a {\em relational marginal polytope} grows only polynomially with the size of the domain. It is quite frequent in applications of SRL that the number of first-order logic formulas is quite small and the domain is large, which is also one of the motivations for so-called {\em lifted inference} algorithms \cite{braz2005lifted,gogate2011probabilistic,broeck2011completeness} which exploit symmetries stemming from the first-order nature of statistical relational models. It is therefore important to understand the properties of relational marginal polytopes, in particular the complexity of the relevant problems as a function of the domain size, which is what we study in this paper. In particular the main contributions of this paper are as follows: (i) We show that if a class of Markov logic networks admits a domain-lifted inference algorithm (in particular an algorithm for computing the partition function of Markov logic networks, with runtime polynomial in the size of the domain) then there is also a domain-lifted algorithm for constructing the respective relational marginal polytopes. This allows us to extend recent domain-liftability results for weight learning of Markov logic networks from \cite{kuzelka.aistats.2019}. (ii) We show that the problem of deciding whether a point is contained in a relational marginal polytope given by a set of first-order logic formulas and a domain of size $n$ is, in general, not domain-liftable. This negative result for containment then also directly translates into a negative result for relational marginal polytope construction (because the number of vertices of the polytope is polynomial in the domain size).

The results in this paper are mostly meant as showing what is and what is not possible {\em theoretically}. Making the results {\em practical} is a different question.

\section{Preliminaries}\label{sec:preliminaries}

\subsection{First Order Logic}

We assume a function-free first-order language defined by a set of constants $\Delta$, a set of variables $\mathcal{V}$ and for each $k\in \mathbb{N}$ a set $\mathcal{R}_k$ of $k$-ary predicates. Variables start with lowercase letters and constants start with uppercase letters. An atom is $r(a_1,...,a_k)$ with $a_1,...,a_k\in \Delta \cup \mathcal{V}$ and $r\in \mathcal{R}_k$. A literal is an atom or its negation. A clause is a universally quantified disjunction of a finite set of literals.  
A clause in which none of the literals contains any variables is called {\em ground}. The set of grounding substitutions of a clause $\alpha$ w.r.t.\ a set of constants $\Delta$ is the set ${\Theta}(\alpha,\Delta)=\{\vartheta_1,...,\vartheta_m\}$ that contains substitutions to all variables occurring in $\alpha$ using constants from $\Delta$. A possible world $\omega$ is represented as a set of ground atoms that are true in $\omega$. The satisfaction relation $\models$ is defined in the usual way: $\omega \models \alpha$ means that the formula $\alpha$ is true in $\omega$. When $\mathbf{x}$ is a list of first-order logic variables then $|\mathbf{x}|$ is used to denote the length of this list.

An important concept that we will need in this paper is the {\em spectrum of a first-order logic sentence}.

\begin{definition}[Spectrum of sentence]
Let $\Phi$ be a first-order logic sentence. Its spectrum is the set of integers $n$ for which $\Phi$ has a model on domain of size $n$.
\end{definition}

\subsection{Complexity Classes ETIME and NETIME}

In this section we briefly describe the necessary results from complexity theory that we will use for proving negative results in this paper. 
\begin{definition}[Complexity classes]
Given a function $f$, the class $\mathbf{DTIME}[f(n)]$ (or $\mathbf{NTIME}[f(n)]$) of decision problems are solvable by $f(n)$ time-bounded deterministic (or non-deterministic) Turing machines. The class $\mathbf{ETIME}$ is defined as $\bigcup_{c>0} \mathbf{DTIME}(2^{cn})$ and the class $\mathbf{NETIME}$ equals to $\bigcup_{c>0} \mathbf{NTIME}(2^{cn})$. 
\end{definition}

The class $\mathbf{NETIME}$ is tightly related to the spectra of first-order logic sentences via the following result of Jones and Selman.

\begin{theorem}[Jones and Selman \citeyear{jones1974turing}]\label{thm:jonesselman}
A set $\mathcal{A} \subseteq \mathbb{N}$ is in $\mathbf{NETIME}$ if and only if $\mathcal{A}$ is the spectrum of a function-free and constant-free sentence $\Phi$.
\end{theorem}

\subsection{Markov Logic Networks}\label{sec:preliminaries:mlns}

A Markov logic network \cite{Richardson2006} (MLN) is a set of weighted first-order logic formulas $(\alpha,w)$, where $w\in \mathbb{R}$ and $\alpha$ is a function-free and quantifier-free first-order formula. The semantics are defined w.r.t.\ the groundings of the first-order formulas, relative to some finite set of constants $\Delta$, called the domain. An MLN $\Phi$ induces the probability distribution over possible worlds $\omega \in \Omega$: 
$
p_{\Phi}(\omega) = \frac{1}{Z} \exp \left(\sum_{(\alpha,w) \in \Phi} w \cdot N(\alpha,\omega)\right),
$
where $N(\alpha, \omega)$ is the number of groundings of $\alpha$ satisfied in $\omega$, and $Z$, called {\em partition function}, is a normalization constant to ensure that $p_{\Phi}$ is a  probability distribution. It turns out to be more convenient to replace $N(\alpha,\omega)$ in the definition of MLNs by 
$Q(\alpha,\omega) = \frac{1}{|\Delta|^{|\textit{vars}(\alpha)|}} \sum_{\vartheta \in \Theta(\alpha,\Delta)} \mathds{1}(\omega \models \alpha\vartheta),$  
where $\Theta(\alpha,\Delta)$ is the set of all grounding substitutions of $\alpha$'s variables using constants from $\Delta$ and $\mathds{1}(\omega \models \alpha\vartheta)$ is the indicator function, which is equal to 1 when $\alpha\vartheta$ is true in the possible world $\omega$. Thus, $Q(\alpha,\omega)$ is the fraction of the groundings of $\alpha$ satisfied in $\omega$. Equivalently, we can say that $Q(\alpha,\omega)$ is the probability that $\omega \models \alpha\varphi$ if we pick $\varphi$ from $\Theta(\alpha,\Delta)$ uniformly at random. With this notation we will write the probability of a possible world $\omega \in \Omega$ as: 
$
p_{\Phi}(\omega) = \frac{1}{Z} \exp \left(\sum_{(\alpha,w) \in \Phi} w \cdot Q(\alpha,\omega)\right).
$

% \begin{note}
% {\em It is possible to define MLNs and relational marginal problems (described in Section \ref{sec:preliminaries:relmarg}) using injective substitutions, which is sometimes more convenient, equally expressive \cite{DBLP:conf/aaai/BuchmanP15}, and which was used for instance in \cite{kuzelka2018relational,kuzelka.aistats.2019}. In this work we do not use the definition of MLNs based on injective substitutions. All results presented in this paper translate straightforwardly to the case with injective substitutions as well.}
% \end{note}

\subsection{Relational Marginal Problems}\label{sec:preliminaries:relmarg}

An alternative way to view a Markov logic network $\Phi = \{ (\alpha_1,w_1), \dots, (\alpha_m,w_m) \}$ is to think of it as a maximum entropy distribution satisfying given marginal constraints $\mathbb{E}[Q(\alpha_i,.)] = \theta_i$.
%Assuming that we are given the values that the formula statistics $Q(\alpha_1,\omega)$, $\dots$, $Q(\alpha_m,\omega)$ should have in expectation 
Assuming we have the expected values of the formula statistics $\mathbb{E}[Q(\alpha, .)]$, we can define the following maximum entropy problem~\cite{kuzelka2018relational}.

\paragraph{Relational Marginal Problem (Formulation):}

\begin{align}
    \min_{\{ P_\omega \colon \omega \in \Omega \}}  \sum_{\omega\in \Omega} P_{\omega} \log{{P_\omega}} \quad \textit{ s.t.}\label{eq:maximum_entropy_criterion} \\
    \forall i = 1,\dots,m: \sum_{\omega \in \Omega} P_\omega \cdot Q(\alpha_i,\omega) = \theta_i\label{eq:maxent_marginal_constraints} \\
\forall \omega \in \Omega : P_{\omega} \geq 0, \sum_{\omega \in \Omega} P_{\okedit{\omega}} = 1\label{eq:maximum_entropy_cnormalization}
\end{align}

\noindent Here, the $P_\omega$'s are the problem's decision variables, each of which represents the probability of one possible world $\omega \in \Omega$. Line (\ref{eq:maximum_entropy_criterion}) is the maximum entropy criterion, which is shown here as the minimization of the negative entropy; Line (\ref{eq:maxent_marginal_constraints}) shows the constraints given by the statistics; and  Line~(\ref{eq:maximum_entropy_cnormalization}) provides the normalization constraints for the probability distribution.

Assuming there exists a feasible solution satisfying $\forall \omega : P_\omega > 0$ (we call such a solution {\em positive}), the optimal solution of the above maximum entropy problem is an MLN 
$P_\omega = \frac{1}{Z} \exp{\left( \sum_{(\alpha_i,w_i) \in \Phi} w_i \cdot Q(\alpha_i,\omega) \right)}$ 
where the parameters $\mathbf{w} = (w_1,\dots,w_m)$ are obtained by maximizing the dual criterion, which also happens to be equivalent to the log-likelihood of the MLN w.r.t.\ a (possibly fictitious) training example $\widehat{\omega}$ that is over the same domain $\Delta$ and satisfies $Q(\alpha_i,\widehat{\omega}) = \theta_i$ for all the formula statistics (we refer to \cite{kuzelka2018relational} for details).

\subsection{Inference Using Weighted Model Counting}\label{sec:WFOMC}

For exact inference in MLNs, one generally needs to be able to compute the partition function $Z = \sum_{\omega \in \Omega} e^{\sum_{\alpha_j} \lambda_j Q_\omega(\alpha_j)}$. 
Computation of the partition function $Z$ can be converted to a {\em weighted first-order model counting problem} (WFOMC).

\begin{definition}[WFOMC \cite{broeck2011completeness}]
Let $w(P)$ and $\overline{w}(P)$ be functions from predicates to real numbers (we call $w$ and $\overline{w}$ {\em weight functions}) and let $\Gamma$ be a first-order theory. Then $\mathbf{WFOMC}(\Phi,w,\overline{w}) =
    \sum_{\omega \in \Omega : \omega \models \Gamma} \prod_{a \in \mathcal{P}(\omega)} w(\textit{Pred}(a)) \cdot \prod_{a \in \mathcal{N}(\omega)} \overline{w}(\textit{Pred}(a))$, 
%\begin{align*}
    % &\mathbf{WFOMC}(\Phi,w,\overline{w}) =\\ 
    % &\sum_{\omega \in \Omega : \omega \models \Gamma} \prod_{a \in \mathcal{P}(\omega)} w(\textit{Pred}(a)) \prod_{a \in \mathcal{N}(\omega)} \overline{w}(\textit{Pred}(a))
%\end{align*}
where $\mathcal{P}(\omega)$ and $\mathcal{N}(\omega)$ denote the positive literals that are true and false in $\omega$, respectively, and $\textit{Pred}(a)$ denotes the predicate of $a$ (e.g. $\textit{Pred}(\textit{friends}(\textit{Alice},\textit{Bob})) = \textit{friends}$).
\end{definition}

To compute the partition function $Z$ using weighted model counting, we may proceed as in \cite{broeck2011completeness}. Let a Markov logic network $\Phi$ be given. Here, for simplicity of exposition, we will assume that the formulas in $\Phi$ do not contain constants (we refer to \cite{broeck2011completeness} for the general case). For every weighted formula $(\alpha_i,\lambda_i) \in \Phi$, where the free variables in $\alpha_i$ are exactly $x_1$, $\dots$, $x_k$, we create a new formula 
%\begin{equation*}
    $\forall x_1,\dots,x_k : \xi_i(x_1,\dots,x_k) \Leftrightarrow \alpha_i(x_1,\dots,x_k)$
%\end{equation*}
where $\xi$ is a new fresh predicate; we denote the resulting set of formulas $\Gamma$. Then we set 
$w(\xi_i) = \exp{\left(|\Delta|^{-|\textit{Vars}(\alpha_i)|} \cdot \lambda_i \right)}$ 
and $\overline{w}(\xi_i) = 1$ and for all other predicates we set both $w$ and $\overline{w}$ equal to 1. It is easy to check that then $\mathbf{WFOMC}(\Gamma,w,\overline{w}) = Z$, which is what we needed to compute.

\subsection{Liftability}

Importantly, there are classes of first-order logic theories for which weighted model counting is polynomial-time. In particular, as shown in \cite{van2014skolemization}, when the theory consists only of first-order logic sentences, each of which contains at most two logic variables, the weighted model count can be computed in time polynomial in the number of elements in the domain $\Delta$ over which the set of possible worlds $\Omega$ is defined. It follows from the translation described in the previous section that this also means that computing the partition function of $2$-variable MLNs can be done in time polynomial in the size of the domain and in the absolute value of the MLN's weights. This is not the case in general when the number of variables in the formulas is greater than two unless P = \#P$_1$\footnote{\#P$_1$ is the set of \#P problems over a unary alphabet.}~\cite{beame2015symmetric}. Within statistical relational learning, the term used for problems that have such polynomial-time algorithms is {\em domain liftability}.
Next, we define what we mean by {\em domain liftability} \cite{broeck2011completeness}.

\begin{definition}[Domain liftability]
An algorithm for computing the partition function $Z$ of an MLN $\Phi = \{ (\alpha_1, \lambda_1), \dots, (\alpha_l, \lambda_l) \}$, \okedit{where each $\lambda_i$ is represented by two numbers\footnote{The restriction on the representation of the weights ensures that the partition function will always be a rational number. Moreover, one can verify that the number of bits needed to represent the partition function will also be polynomial in the number of bits needed to represent the numbers $a_i$, $b_i$ and in the domain size $|\Delta|$.} $a_i, b_i \in \mathbb{N}$ as $\lambda_i =  \ln a_i - \ln b_i$}, is said to be {\em  domain-lifted} if it runs in time polynomial in the size of the domain $\Delta$ and in \okedit{the number of bits needed to encode the numbers $a_i$ and $b_i$}. A class of MLNs is said to be {\em domain-liftable} if there is a domain lifted algorithm for computing the partition function $Z$ for MLNs from this class.
\end{definition}

\noindent The definition that we use here differs slightly from the original definition by Van den Broeck \cite{broeck2011completeness} in that it also requires lifted algorithms to depend polynomially on the \okedit{size of the representation of the} formulas' weights. A justification for this alternative definition follows from the work of Jaeger \cite{jaeger} (Section 4.2). In particular, all existing domain-lifted inference algorithms are also domain-lifted according to our definition. Another small technical difference is that we define domain-liftability directly in terms of complexity of computing the partition function $Z$.

\section{Relational Marginal Polytopes}\label{sec:polytopes}

Now we can finally define relational marginal polytopes \cite{kuzelka2018relational} which are the main focus of this paper. These represent the expected values for the vectors of statistics of the given formulas that are possible.

\begin{definition}[Relational marginal polytope \cite{kuzelka2018relational}]
Let $\Omega$ be the set of possible worlds on domain $\Delta$ and $\Phi = (\alpha_1, \dots, \alpha_m)$ be a list of formulas. We define the relational marginal polytope $\mathbf{RMP}(\Phi,\Delta)$ w.r.t.\ $\Phi$ as $\mathbf{RMP}(\Phi,\Delta) = \{ (x_1,\dots,x_m) \in \mathbb{R}^m : \exists \mbox{ dist.\ on } 
     \Omega\; { s.t. } \;\mathbb{E}[Q(\alpha_1,\omega)] = x_1 \wedge \dots \wedge \mathbb{E}[Q(\alpha_m,\omega)] = x_m \}.$
\end{definition}

\noindent The relational marginal polytope w.r.t.\ a given list of formulas $(\alpha_1$, $\dots$, $\alpha_m)$ can be equivalently defined as the convex hull of the set 
$\{ (Q(\alpha_1,\omega),\dots,Q(\alpha_m, \omega)) : \omega \in \Omega \}$.

Next, we illustrate the fact that relational marginal polytopes can be quite complex using the following example from \cite{kuzelka.aistats.2019}. In fact, we show later in this paper that, in general, the problem of deciding whether a point belongs to a relational marginal polytope is not domain-liftable.

\begin{example}\label{ex:densities}
Consider the formulas $\alpha = \textit{friends}(x_1,x_2)$ and $\beta = \textit{friends}(x_1,x_2) \wedge \textit{friends}(x_2,x_3) \wedge \textit{friends}(x_3,x_1)$. Let $\Delta = \{ C_1, \dots, C_{100} \}$ be the set of domain elements and $\Omega$ be the respective set of possible worlds over the first-order language given by the predicate $\textit{friends}/2$ and the constants from $\Delta$. The possible worlds $\omega \in \Omega$ may be thought of as representing social networks. Then $Q(\alpha,\omega)$ corresponds to the ``frequency'' of friendships in the network and $Q(\beta,\omega)$ to the ``frequency'' of friendship-triangles. We can then see easily why there is, for instance, no distribution with $\mathbb{E}[Q(\alpha,\omega)] = 0$ and $\mathbb{E}[Q(\beta,\omega)] = 0.5$ (as graphs without edges cannot have a positive number of triangles). Hence, the point $(0,1)$ will not be contained in the relational marginal polytope. In general relational marginal polytopes may be quite complicated objects as, for instance, in this case to construct the respective polytope we would also need to also find the ``extremal'' directed graphs with maximum number of triangles with a constrained number of edges etc.
\end{example}

Next, we define what it means for a point to be in the $\eta$-interior of a polytope.

\begin{definition}[Interiority]
Let $\eta > 0$, $\mathbf{P}$ be a polytope and $A^= \mathbf{x} = \mathbf{c}$ be the maximal linearly independent system of linear equations that hold for the vertices of $\mathbf{P}$. A point $\theta$ is said to be in the $\eta$-interior of $\mathbf{P}$ if 
$\{\theta' | A^= \theta' = \mathbf{c}, \| \theta' - \theta \|_2 \leq \eta \} \subseteq \mathbf{P}. $
\end{definition}

\noindent We need to consider the system of linear equations $A^= \mathbf{x} = \mathbf{c}$ in the definition of interiority because the polytope may live in a lower-dimensional subset of the given space. Our definition of interiority is also often called {\em relative interiority} in the literature. 

% If the vector of formula statistics' estimates $\theta$ is in the $\eta$-interior of the respective relational marginal polytope for some $\eta > 0$, then there always exists a positive distribution satisfying the marginal constraints given by the statistics.

% \section{Simple Properties of Relational Marginal Polytopes}

% \begin{proposition}\label{prop:simple1}

% \end{proposition}

\section{A Lifted Reduction to Partition Functions of MLNs}\label{sec:lifted}
In this section, we show that computing relational marginal polytopes can be reduced to computing the partition functions of some MLNs.
% In order to prove this result, we first define several notations. 
Instead of $Q(\alpha,\omega)$, we deal with $N(\alpha,\omega)$ in this section, since there is a simple one-to-one mapping $Q(\alpha,\omega) = {|\Delta|^{-|vars(\alpha)|}}\cdot N(\alpha,\omega)$ and $N(\alpha,\omega)$ only uses integers. 
We call the convex hull of the set $\{(N(\alpha_1,\omega),\ldots, N(\alpha_m,\omega)):\omega\in \Omega\}$ the \emph{integer relational marginal polytope} $\mathbf{IRMP}(\Phi,\Delta)$. 
Let $P_i := \{N(\alpha_i,\omega):\omega\in \Omega\}$. 
We can see that $\{ (N(\alpha_1,\omega),\dots,N(\alpha_m, \omega)) : \omega \in \Omega \}$ is a subset of $\bigtimes_{i=1}^m P_i$. 
The algorithm described in the following theorem is in polynomial time. Even though it would not be efficient in practice, our goal here is to show a reduction, instead of devising a fast algorithm.  

\begin{theorem}\label{theorem:partition_function}
Let $\Phi = (\alpha_1,\dots,\alpha_m)$ be a list of first-order logic formulas. If the class of MLNs $\{(\alpha_i,\lambda_i)\}$ is domain-liftable 
then computing the facets of the relational marginal polytope $\mathbf{RMP}(\Phi,\Delta)$ \okedit{can be done in time polynomial in $|\Delta|$}. %is also domain-liftable.
\end{theorem}

\begin{proof}
% Let $n$ be the domainsize and 
% Let $k_i$ be the arity of $\alpha_i$. 
We first observe that, for any world $\omega$, $N(\alpha_i,\omega)$ can only take value in $[r_i] = \{0,1,\ldots, r_i\}$ where $r_i := \max P_i = O(|\Delta|^{|vars(\alpha_i)|}).$  %O(n^{k_i})$. 
Hence, $N(\Phi,\omega)$ can only take value in $\bigtimes_{i=1}^m [r_i]$, which are integer points (vectors) in an $m$-dimensional space,  whose $i$-th entries represent $N(\alpha_i,\omega)$. 

\okedit{If $\mathbf{IRMP}(\Phi,\Delta)$ is full-dimensional, each of its facets} corresponds to a halfspace, which is of the form $\sum_i a_i x_i \le b$, where $x_i$ is the $i$-th coordinate, and can be written as $\Vec{a} \cdot \Vec{x} \le b$. \okedit{If $\mathbf{IRMP}(\Phi,\Delta)$ is not full-dimensional, each of its facets corresponds to an intersection of such half-spaces.}
To determine these halfspaces, we first list all possible normal vectors $\Vec{a}$, and then find $b$ for each normal vector. 

\paragraph{Enumerate all possible $\Vec{a}$'s}
Note that every vertex of the $\mathbf{IRMP}$ 
% (which cannot be represented as a convex combination of two different points in the polytope) 
is an integer vector in $\bigtimes_{i=1}^m [r_i]$. 
%Every facet of this $\mathbf{IRMP}$ passes through at least $m$ points in $\bigtimes_{i=1}^m [r_i]$. 
We list all possible linearly-independent $m$-tuples $(p^{(1)},p^{(2)},\ldots,p^{(m)})$ of $\bigtimes_{i=1}^m [r_i]$, where every $p^{(j)}$ is a point in $\bigtimes_{i=1}^m [r_i]$. 
For every tuple, we can efficiently compute the perpendicular vector of the hyperplane \okedit{that} passes every point in the tuple. For example, we subtract the coordinates of one of the points from all of the others and then compute their generalized cross product to obtain the perpendicular vector $\Vec{v}$. To enumerate every possible $\Vec{a}$, both $\Vec{v}$ and $-\Vec{v}$ need to be taken into account. In this way, every coordinate of every $\Vec{a}$ is an integer (which is not necessary but makes the next step easier). 
% Every facet of this polytope corresponds to a hyperplane that passes through some points in $\bigtimes_{i=1}^m [r_i]$. We can enumerate all these possible hyperplanes by considering every vector in $\bigtimes_{i=1}^m [r_i]$ except $(0,0,\ldots,0)$:  
% Every vector $v \in \bigtimes_{i=1}^m [r_i]$ defines $2\ell$ hyperplanes, where $\ell$ is the number of non-zero values in the vector $v$. Let $V := \prod_{i:v_i\neq 0} v_i$. The $2\ell$ hyperplanes defined by $v$ are the set of points $(x_1,x_2,\ldots,x_m)$ satisfying $\sum_{i:v_i\neq 0} a_i x_i = b$ where $a_i = \pm V/v_i$ (for each $i$, $a_i$ can be either positive or negative, so there are $2\ell$ hyperplanes) and $b$ will be determined by computing the partition function of some MLN in the next paragraph. 
% using weighted first-order model counting.
% This can be done because every possible facet is parallel to some hyperplane that passes through integer points on the axes. The absolute values of these integer points on the axes together can be written as a vector in $\bigtimes_{i=1}^m [r_i]$ since the integer points on every axis $i$ have at most one non-zero value, which is at position $i$. Note that if a facet is parallel to some axes, then this facet has also been considered by checking a vector that has $0$ on those positions. 
It is not difficult to see that the number of $\Vec{a}$'s that we collect is polynomial in $|\Delta|$ and every entry $a_i$ is polynomial in $|\Delta|$ as well. 

% Now we focus on the case where every $a_i$ is non-negative, and there exists a similar to deal with the cases in which some $a_i$ is negative. 
\paragraph{Find $b$ for each $\Vec{a}$} For every vector $\Vec{a}$, %, we do the same:
to determine the integer $b$ for the halfspace $\sum_{i} a_i x_i \le b$, 
% This is done by applying weighted first-order model counting: For every logical formula $\alpha_i$, we create a new atom $f_i(\text{Vars}(\alpha_i))\equiv\alpha_i.$ The weight functions of the new atom are $w(f_i) =|\Omega|^{2a_i}$ and $\bar{w}(f_i)=1$ if $a_i$ is positive, or else $w(f_i) = 1$ and $\bar{w}(f_i)= |\Omega|^{-2a_i}$. 
% This is done 
we construct a Markov logic network $\{(\alpha_i,  2a_i\ln |\Omega|): 1\le i\le m\}$ and then compute the partition function of this MLN. 
% We describe in the follows for the case where all $a_i$ are positive, and the case where some $a_i$ are negative can be dealt with in a similar way but the corresponding element in the MLN becomes $(\neg\alpha_i,-2a_i\ln |\Omega|)$. 
If for every possible world $\omega$, $\sum_i a_i N(\alpha_i,\omega) \le b$, then the partition function is $\le |\Omega|^{2b+1}$ since there are $|\Omega|$ possible worlds. If there exists a world $\omega$ such that $\sum_i a_i N(\alpha_i,\omega) > b$, then the partition function $\ge |\Omega|^{2b+2}$, since every $a_i$ is an integer.
% If for every possible world $\omega$, $\sum_i a_i N(\alpha_i,\omega) \le b$, then $\mathbf{WFOMC}(\textcolor{red}{need to define this properly}) \le |\Omega|^{2b+1}$ since there are $|\Omega|$ possible worlds. If there exists a world $\omega$ such that $\sum_i a_i N(\alpha_i,\omega) > b$, then $\mathbf{WFOMC}(\textcolor{red}{need to define this properly}) \ge |\Omega|^{2b+2}$.
Therefore, one can easily find the smallest integer $b$ such that there does not exist any world satisfying $\sum_i a_i N(\alpha_i,\omega) > b$. 

After the steps above, we obtain a set of linear inequality constraints (halfspaces). We can apply Lemma 2 %\ref{lemma:minimalset} 
(in the appendix) to get a minimal set of constraints that correspond to the facets of the $\mathbf{IRMP}$, which can be transformed into the facets of $\mathbf{RMP}$ efficiently. 
\end{proof}

Here we would like to note that the techniques for lifted linear programming developed in the work \cite{mladenov,DBLP:journals/ai/KerstingMT17} are not directly relevant to our work. In our case, all the involved linear programs are already small (are polynomial in the size of the domain).

% \begin{lemma}
% Let $\Phi = (\alpha_1,\dots,\alpha_m)$ be a list of first-order logic formulas and $n$ the domainsize. Given a direction (a vector) in the $m$-dimensional space, 
% \end{lemma}

\subsection{A Corollary for Lifted Weight Learning of Markov Logic Networks}

Recently, it has been shown in \cite{kuzelka.aistats.2019} that weight learning of Markov logic networks based on maximization of log-likelihood is domain-liftable for the $2$-variable fragment of Markov logic networks. Previously, it had been shown in \cite{van.haaren.mlj} that computing the gradients of log-likelihood is domain-liftable for liftable fragments of Markov logic networks. What the work \cite{kuzelka.aistats.2019} added to this was to show that the whole weight learning problem is also domain-liftable, not just the procedure that computes the gradients. This is formally stated in the next theorem from \cite{kuzelka.aistats.2019}.

\begin{theorem}[Theorem 11 in \cite{kuzelka.aistats.2019}]\label{thm:mainaistats}
Let $\Phi = \{\alpha_1,\dots,\alpha_l\}$ be a set of quantifier-free first-order logic formulas, each with at most 2 variables. Let $\Phi_0$ be a set of universally quantified first-order logic sentences, each also with at most 2 variables. Let $\Omega_{\Phi_0}$ be the set of models of $\Phi_0$ over a given domain $\Delta$. Let $\widehat{\omega} \in \Omega$ be a training example. Then there is an algorithm which finds weights of the MLN $\mathcal{M}$ given by formulas $\Phi$ such that the log-likelihood of $\mathcal{M}$ given the training example $\widehat{\omega}$ is within $\varepsilon$ of the optimum. The algorithm runs in time polynomial in $|\Delta|$, $1/\varepsilon$ and $1/\eta$ where $\eta$ is the interiority of the vector $Q_{\widehat{\omega}}(\Phi)$ in the relational marginal polytope $\textit{RMP}(\Phi,\Omega_{\Phi_0})$. 
\end{theorem}

The above theorem is restricted only to the $2$-variable fragment of Markov logic networks but a bit larger tractable fragments were also studied in the literature, e.g. \cite{DBLP:conf/nips/KazemiKBP16}. The only reason why the proof of Theorem \ref{thm:mainaistats} needs the assumption that the Markov logic network is from the $2$-variable fragment is that the weight learning algorithm needs the respective relational marginal polytope and the work \cite{kuzelka.aistats.2019} only gives a domain-lifted algorithm for computing relational marginal polytopes for the $2$-variable fragment. However, the result about domain-liftability that we presented in this section guarantees that the construction of relational marginal polytopes is always domain-liftable for Markov logic networks with domain-liftable inference. 
Hence we get a strengthening of the results from \cite{kuzelka.aistats.2019}. 

\begin{corollary}
The positive result stated in Theorem \ref{thm:mainaistats} holds for all Markov logic networks for which computing the partition function $Z$ is domain-liftable, i.e.\ not just for the $2$-variable fragment.
\end{corollary}

\section{Complexity of the Containment Test}\label{sec:containment}

In this section we study computational complexity of deciding if a point is contained in the relational marginal polytope given by a fixed list of first-order logic formulas and a domain $\Delta$. In particular we are interested in the complexity w.r.t.\ the domain size as a parameter, i.e.\ in domain-liftability of the problem. We prove a negative result showing that, unless $\mathbf{ETIME} = \mathbf{NETIME}$, there is no algorithm for deciding if a point $\theta = \okedit{(p_1/q_1,\dots,p_k/q_k)}$ is contained inside $\mathbf{RMP}(\Phi, \Delta)$ running in time polynomial in $|\Delta|$, \okedit{$\sum_{i=1}^k p_i$ and $\sum_{i=1}^k q_i$}.

The proof of the negative result borrows ideas from the work of Jaeger \cite{jaeger}. In particular, we will need the next two results.

\begin{proposition}[Corollary 3.4 in \cite{jaeger}]\label{corollary:jaeger}
If $\mathbf{ETIME} \neq \mathbf{NETIME}$ then there exists a function-free and constant-free first-order sentence $\Phi$ such that $\{\textit{unary}(n) | n \in \textit{spec}(\Phi) \}$ is not recognized in deterministic polynomial time.
\end{proposition}
% Theorem 3.3
% (Jones and Selman 1972) 
% Corollary 3.4
% If NETIME ̸= ETIME, then there exists a first-order sentence φ, such that {un(n) | n ∈ spec(φ)} is not recognized in deterministic polynomial time.

\begin{lemma}[Proposition 4.2 in \cite{jaeger}]\label{lemma:jaeger}
Let $\Phi(\mathbf{x})$ be a quantifier-free first-order logic formula (possibly containing constants and function symbols). Let $\mathcal{S}$ be the set of relation symbols and $\mathcal{S}^F$ the set of function symbols and constants contained in $\Phi(\mathbf{x})$. Let $\mathcal{S}^+$ be a set of new relation symbols that for every $k$-ary $f \in \mathcal{S}^F$ contains a $k + 1$-ary $R^f$ (constant symbols are treated as $0$-ary function symbols). Let $\textit{Func}$ be a set of sentences such that for every $f \in \mathcal{S}^F$ it contains the following first-order logic sentences:
\begin{align}
    \forall \mathbf{x}, y, y' : R^f(\mathbf{x},y) \wedge R^f(\mathbf{x},y') \Rightarrow y = y' \\
    \forall \mathbf{x} \exists y : R^f(\mathbf{x},y). \label{eq:rf}
\end{align}
Then there exists a formula $\Phi^+(\mathbf{x},\mathbf{z})$ without quantifiers, function symbols and constants such that the following are equivalent for all $n$:
\begin{enumerate}
    \item there exists an $\omega \in \Omega_{\mathcal{S} \cup \mathcal{S}^F}$ \okedit{on domain $\Delta = \{1,\dots,n\}$} such that $\omega \models \forall \mathbf{x} : \Phi(\mathbf{x})$,
    \item there exists an $\omega^+ \in \Omega_{\mathcal{S} \cup \mathcal{S}^+}$ \okedit{on domain $\Delta = \{1,\dots,n\}$} such that $\omega^+ \models \textit{Func} \wedge \forall \mathbf{x}, \mathbf{z} : \Phi(\mathbf{x},\mathbf{z})$.
\end{enumerate}
\end{lemma}

Next, we state and prove our main negative result.

\begin{theorem}
There exists a list of first-order logic formulas $\Gamma = (\alpha_1,\dots,\alpha_k)$ such that the following holds. If $\mathbf{ETIME} \neq \mathbf{NETIME}$ then there is no algorithm for deciding if a point $\theta = (p_1/q_1,\dots,p_k/q_k)$, where $p_i,q_i \in \mathbb{N}$ is contained in the relational marginal polytope $\mathbf{RMP}(\Gamma, \Delta)$ that runs in time polynomial the size of the domain $\Delta$ and in $\sum_{i=1}^k p_i$ and $\sum_{i=1}^k q_i$. The hardness result holds even if the point $\theta$ is guaranteed not to be equal to a vertex of the polytope.
\end{theorem}

\noindent {\em Proof idea.}
We use the techniques which were devised by Jaeger to prove lower bounds on the complexity of weighted model counting. Specifically, Proposition \ref{corollary:jaeger} tells us that there exists a function-free first-order logic sentence $\Phi^*$ whose spectrum is not recognizable in time polynomial in the size of the domain (assuming $\mathbf{ETIME} \neq \mathbf{NETIME}$). 
Lemma \ref{lemma:jaeger} allows us to convert that sentence to a relatively manageable form. The conversion is done by first Skolemizing the sentence $\Phi^*$ (which introduces function symbols) and then using Lemma \ref{lemma:jaeger} to get rid of the function symbols. Although Lemma \ref{lemma:jaeger} introduces new existential quantifiers, these may then appear only in (\ref{eq:rf}), which we will be able to deal with. Then what we need to show is that if we could decide membership of points in relational marginal polytopes efficiently then we could also efficiently decide satisfiability of $\Phi^*$ on domains of a given size. This is not completely straightforward because, first, the formulas in $\Gamma$ cannot have quantifiers and, second, the points should not correspond to statistics of individual possible worlds so that we could still guarantee that these points do not correspond to vertices of the polytopes. Hence, we need to select the point $\theta$ carefully; we will place it quite (polynomially) close to a potential vertex of the polytope. Finally, we  have to show that the point $\theta$ is indeed inside the polytope for domain of size $n$ if and only if the sentence $\Phi^*$ has a model on a domain of size $n$.

\begin{proof}
Let $\Phi^*$ be a first-order logic sentence with non-polynomial-time spectrum whose existence follows from Proposition \ref{corollary:jaeger}. Let $\forall \mathbf{x} : \Phi(\mathbf{x})$ be a Skolemization of $\Phi^*$ and let 
$\Phi^{\textit{Skol}}(\mathbf{x}) = \textit{Func} \wedge \forall \mathbf{x},\mathbf{z} : \Phi(\mathbf{x},\mathbf{z})$  
be constructed from $\Phi(\mathbf{x})$ as in Lemma \ref{lemma:jaeger}. Let us write 
\begin{multline*}
\textit{Func} = (\forall \mathbf{x}, y, y' : \beta_1(\mathbf{x},y,y')) \wedge (\forall \mathbf{x} \exists y : \gamma_1(\mathbf{x},y)) \wedge \dots \\
\wedge (\forall \mathbf{x}, y, y' : \beta_l(\mathbf{x},y,y')) \wedge (\forall \mathbf{x} \exists y : \gamma_l(\mathbf{x},y))    
\end{multline*}
where
%\begin{align*}
    $\beta_i(\mathbf{x},y,y') = R^{f_i}(\mathbf{x},y) \wedge R^{f_i}(\mathbf{x},y') \Rightarrow y = y'$ and
    $\gamma_i(\mathbf{x},y) = R^{f_i}(\mathbf{x},y)$. 
%\end{align*}
Finally we define 
%\begin{align*}
    $\Phi' = A_{\Phi'}(\mathbf{x}, \mathbf{z}) \wedge \Phi(\mathbf{x},\mathbf{z}),$
%\end{align*}
and for every $1 \leq i \leq l$ we define
\begin{align*}
    \beta_i' =& A_{\beta_i}(\mathbf{x},y,y') \wedge (R^{f_i}(\mathbf{x},y) \wedge R^{f_i}(\mathbf{x},y') \Rightarrow y = y'), \\
    \gamma_i' =& A_{\gamma_i}(\mathbf{x},y) \wedge R^{f_i}(\mathbf{x},y).
\end{align*}
where $A_{\Phi'}$, $A_{\beta_i}$'s and $A_{\gamma_i}$'s are auxiliary relations which will be useful in turn. Then we construct the list of formulas 
$\Gamma = (\Phi', \beta_1', \gamma_1', \ldots, \beta_l', \gamma_l').$ 
Next all we need to do is to select a point $\theta$ such that the respective polytope-membership query would allow us to decide if $\Phi^*$ has a model on a domain of size $n$. We define $\Delta = \{1,2,\dots,n\}$, $\varepsilon = \frac{1}{l} \cdot \frac{1}{2l+2} \cdot \frac{1}{|\Delta|^{|\mathbf{x}|+|\mathbf{z}|+2}} \cdot \frac{1}{|\Delta|^{|\mathbf{x}|+1}+1},$ and $\theta = \left(1-\varepsilon, 1-\varepsilon, \frac{1-\varepsilon}{|\Delta|}, \ldots,  1-\varepsilon, \frac{1-\varepsilon}{|\Delta|} \right)$. 
% \begin{align*}
    %, \\
    %\varepsilon =& \frac{1}{l} \cdot \frac{1}{2l+2} \cdot \frac{1}{|\Delta|^{|\mathbf{x}|+|\mathbf{z}|+2}} \cdot \frac{1}{|\Delta|^{|\mathbf{x}|+1}+1}, \\
    % \theta =& \left(1-\varepsilon, 1-\varepsilon, \frac{1-\varepsilon}{|\Delta|}, \ldots,  1-\varepsilon, \frac{1-\varepsilon}{|\Delta|} \right).
% \end{align*}
Here we note that $1/\varepsilon$ is polynomial in the size of the domain and so are  the representations of $\Delta$, $\varepsilon$ and $\theta$. It also follows from the way $\varepsilon$ is selected that $\theta$ cannot be a vertex of the polytope.

Next, we show that $\theta \in \mathbf{RMP}(\Gamma,\Delta)$ if and only if $\Phi^*$ has a model on a domain of size~$n$.

($\Rightarrow$) We will show that if $\theta \in \mathbf{RMP}(\Gamma,\Delta)$ then $\forall \mathbf{x} : \Phi^{\textit{Skol}}(\mathbf{x})$ has a model on a domain of size $n$. Since, by Lemma \ref{lemma:jaeger}, $\forall \mathbf{x} : \Phi^{\textit{Skol}}(\mathbf{x})$ has such a model if and only if $\Phi^*$ has a model on a domain of size $n$, it will also follow that if $\theta \in \mathbf{RMP}(\Gamma,\Delta)$ then $\Phi^*$ has a model on a domain of size $n$. 

First, using Carath\'{e}odory's theorem, we know that $\theta$ can be written as a convex combination: 
\begin{equation}\label{eq:convexcombination}
\theta = \sum_{j=1}^{2l+2} a_j \cdot Q(\Phi,\omega_j).
\end{equation}
Now we want to show that at least one of the possible worlds $\omega_i$ in (\ref{eq:convexcombination}), which we denote as $\widetilde{\omega}$, must satisfy the following three conditions:
\begin{enumerate}
    \item $Q(\Phi',\widetilde{\omega}) = 1$,
    \item $Q(\beta_i',\widetilde{\omega}) = 1$ for all $1 \leq i \leq l$,
    \item $Q(\gamma_i',\widetilde{\omega}) = \frac{1}{|\Delta|}$ for all $1 \leq i \leq l$.
\end{enumerate}
Let us denote $\theta' = (1,1,1/|\Delta|,1,1/|\Delta|,\dots,1,1/|\Delta|)$. Then the above conditions can be also written as $Q(\Gamma,\widetilde{\omega}) = \theta'$.
We start by showing that the first two conditions hold. Suppose, for contradiction, that one of these two conditions does not hold. Then $\|\theta'-\theta \|_\infty \geq \frac{1}{2l+2} \cdot \frac{1}{|\Delta|^{|\mathbf{x}|+|\mathbf{z}|+2}}$ which follows from the fact that $a_j \geq \frac{1}{2l+2}$ must hold for at least one of the $a_j$'s in (\ref{eq:convexcombination}). However, we also have $\|\theta'-\theta \|_\infty \leq \varepsilon$ and $\varepsilon < \frac{1}{2l+2} \cdot \frac{1}{|\Delta|^{|\mathbf{x}|+|\mathbf{z}|+2}}$, hence we have a contradiction. Therefore there must be at least one $\widetilde{\omega}$ that satisfies the first two conditions and there may, in fact, be multiple such worlds. Hence, $\widetilde{\omega}$ is a possible world on a domain of size $n$ which is a model of the first-order logic sentences $\forall \mathbf{x}, \mathbf{z} : \Phi(\mathbf{x},\mathbf{z})$, $\forall \mathbf{x}, y, y' : \beta_1(\mathbf{x},y,y')$, $\dots$, $\forall \mathbf{x}, y, y' : \beta_l(\mathbf{x},y,y')$.

To show that at least one possible world from (\ref{eq:convexcombination}) must satisfy all three conditions, we will need to strengthen the above argument. We define $S_1 = \sum_{j \in I_1} a_j$ where $I_1$ is the set of indices of those possible worlds $\omega_j$ from (\ref{eq:convexcombination}) that satisfy the first two conditions, and $S_2 = 1-S_1$. Using similar reasoning as above, we must have 
$S_2 \cdot \frac{1}{l+1} \cdot \frac{1}{|\Delta|^{|\mathbf{x}|+|\mathbf{z}|+2}} \leq \varepsilon $
as otherwise we would have $\|\theta-\theta' \|_\infty > \varepsilon$, which can be seen as follows. If one of the sentences $\forall \mathbf{x}, \mathbf{z} : \Phi(\mathbf{x},\mathbf{z})$, $\forall \mathbf{x}, y, y' : \beta_1(\mathbf{x},y,y')$, $\dots$, $\forall \mathbf{x}, y, y' : \beta_l(\mathbf{x},y,y')$ is violated then $\|Q(\Gamma,\omega_j) - \theta' \|_\infty \geq 1/|\Delta|^{|\mathbf{x}|+|\mathbf{z}|+2}$. We then obtain the inequality when we realize that each of the possible worlds $\omega_j$, where $j \in I_2$, must violate at least one of the $l+1$ formulas.

Hence it follows from the above that the next inequality must hold
\begin{equation}\label{eq:s21}
    S_2 \leq (l+1) \cdot |\Delta|^{|\mathbf{x}|+|\mathbf{z}|+2} \cdot \varepsilon.
\end{equation}

\noindent Next, we can notice that if $Q(\beta_i',\widetilde{\omega}) = 1$ then it must also hold $Q(\gamma_i',\widetilde{\omega}) \leq \frac{1}{|\Delta|}$. Suppose, for contradiction, that none of the possible worlds $\widetilde{\omega} \in \{\omega_{j} | j \in I_1 \}$ satisfies the third condition. Then it must be the case for each of these possible worlds that $Q(\gamma_i',\widetilde{\omega}) \leq \frac{|\Delta|^{|\mathbf{x}|}-1}{|\Delta|^{|\mathbf{x}|+1}}$ for at least one $\gamma_i'$. We need the following to hold for every $\gamma_i'$ from $\Gamma$: 
$
\sum_{j \in I_2} a_i \cdot Q(\gamma_i',\omega_j) + \sum_{j \in I_1} a_j \cdot Q(\gamma_i', \omega_j) = \frac{1-\varepsilon}{|\Delta|}
$
and we know that $\sum_{j \in I_2} a_i \cdot Q(\gamma_i',\omega_j) \leq S_2$ and $\sum_{j \in I_1} a_i \cdot Q(\gamma_i',\omega_j) \leq (1-S_2) \cdot \frac{|\Delta|^{|\mathbf{x}|}-1/l}{|\Delta|^{\mathbf{x}+1}}$ for at least one $\gamma_i'$.
It follows that we need (at the very least) the inequality 
%\begin{align*}
    %S_2 \geq \frac{1}{|\Delta|^{|\mathbf{x}|}} \cdot (1-S_2) - \frac{\varepsilon}{|\Delta|}.
    $S_2 + (1-S_2) \cdot \frac{|\Delta|^{|\mathbf{x}|}-1/l}{|\Delta|^{|\mathbf{x}|+1}} \geq \frac{1-\varepsilon}{|\Delta|}$
%\end{align*}
to be true. 
After simple algebraic manipulations this yields the inequality
\begin{equation}\label{eq:s22}
    S_2 \geq \frac{1/l-\varepsilon \cdot |\Delta|^{|\mathbf{x}|}}{|\Delta|^{|\mathbf{x}|+1}-|\Delta|^{|\mathbf{x}|}+1/l} \geq \frac{1/l-\varepsilon \cdot |\Delta|^{|\mathbf{x}|}}{|\Delta|^{|\mathbf{x}|+1}+1}
\end{equation}
Plugging $\varepsilon$ into (\ref{eq:s21}) yields
\begin{equation}
    S_2 \leq \frac{1}{2} \cdot \frac{1}{l} \cdot \frac{1}{|\Delta|^{|\mathbf{x}|+1}+1}
\end{equation}
whereas plugging $\varepsilon$ into (\ref{eq:s22}) yields
\begin{align*}
    S_2 & \geq \frac{\frac{1}{l}-\frac{1}{2l+2} \cdot \frac{1}{l} \cdot \frac{1}{|\Delta|^{|\mathbf{z}|+2}} \cdot \frac{1}{|\Delta|^{|\mathbf{x}|+1}+1}}{|\Delta|^{|\mathbf{x}|+1}+1} \\& \geq \frac{3}{4} \cdot \frac{1}{l} \cdot \frac{1}{|\Delta|^{|\mathbf{x}|+1}+1}.
\end{align*}
So we have arrived at a contradiction. It follows that at least one of the $\widetilde{\omega}$ must simultaneously satisfy conditions 1, 2 and 3. We have thus shown that if $\theta \in \mathbf{RMP}(\Gamma,\Delta)$ then $\Phi^*$ has a model on a domain of size $n$ which is what we needed to show.

% \textcolor{red}{under construction...}
% %First, it follows from the definition of the relational marginal polytope that if $\theta \in \mathbf{RMP}(\Gamma,\Delta)$ then there must be a possible world $\widetilde{\omega}$ such that $Q(\Gamma,\widetilde{\omega}) \succeq \theta$ where $\succeq$ denotes component-wise comparison of vectors (i.e., for vectors $\theta'$ and $\theta''$, $\theta' \succeq \theta''$ holds if and only if for all components of the two vectors we have $[\theta']_i \geq [\theta'']_i$, where $[\theta'']_i$ denotes the $i$-th component of the vector).
% %Second, it follows from the definition of the statistic $Q(\Gamma,\omega)$ that its values must lie on a grid with points that are no closer to each other in a single coordinate than $1/|\Delta|^M$ where $M$ is the maximum number of variables in any of the first-order logic formulas contained in $\Gamma$. 
% %It follows that $\widetilde{\omega}$ must satisfy $Q(\Phi',\widetilde{\omega}) = 1$, $Q(\beta_i',\widetilde{\omega}) = 1$ for all $1 \leq i \leq l$, and consequently also $Q(\beta_i,\widetilde{\omega}) = 1$ for all $1 \leq i \leq l$. 
%  To finish the proof of this direction of the implication, we need to show that $\widetilde{\omega}$ must also be a model of the first-order logic sentences $\forall \mathbf{x} \exists y : \gamma_i(\mathbf{x},y)$ for all $1 \leq i \leq l$. Again

($\Leftarrow$) If $\Phi^*$ has a model on a domain of size $n$ then so does $\forall \mathbf{x} : \Phi^{\textit{Skol}}(\mathbf{x})$. Hence, let $\omega^*$ be such a model of $\forall \mathbf{x} : \Phi^{\textit{Skol}}(\mathbf{x})$ on a domain of size $n$. Let $\omega^+$ be a possible world on the set of relations of $\omega^*$ extended by the $A_{\Phi'}$ and $A_{\beta_i}$ relations. Let $\omega^+$ agree with $\omega^*$ on all the relations from $\omega^*$ and let $\omega^+$ satisfy $\omega^+ \models (\forall \mathbf{x}, \mathbf{z} : A_{\Phi'}(\mathbf{x},\mathbf{z}))$, 
$\omega^+ \models (\forall \mathbf{x}, y, y' : A_{\beta_i}(\mathbf{x},y,y'))$ and 
$\omega^+ \models (\forall \mathbf{x}, y : A_{\gamma_i}(\mathbf{x},y))$, 
for all $1 \leq i \leq l$. Then it is not difficult to check that 
$Q(\Gamma,\omega^+) = \left(1,1, \frac{1}{|\Delta|}, 1, \frac{1}{|\Delta|}, \dots, 1,\frac{1}{|\Delta|} \right).$ 
Next let $\omega^-$ be a possible world constructed from $\omega^*$ in almost the same way as $\omega^+$ but this time satisfying $\omega^- \models (\forall \mathbf{x}, \mathbf{z} : \neg A_{\Phi'}(\mathbf{x},\mathbf{z}))$, 
$\omega^- \models (\forall \mathbf{x}, y, y' : \neg A_{\beta_i}(\mathbf{x},y,y'))$ and 
$\omega^- \models (\forall \mathbf{x}, y : \neg A_{\gamma_i}(\mathbf{x},y))$, 
for all $1 \leq i \leq l$. Then we have
$Q(\Gamma,\omega^-) = \left(0, 0, 0, 0, 0, \dots, 0, 0 \right).$
Since, by definition, $Q(\Gamma,\omega^+) \in \mathbf{RMP}(\Gamma,\Delta)$ and $Q(\Gamma,\omega^-) \in \mathbf{RMP}(\Gamma,\Delta)$ and since $\theta$ is a convex combination of $Q(\Gamma,\omega^+)$ and $Q(\Gamma,\omega^-)$, it follows that $\theta \in \mathbf{RMP}(\Gamma,\Delta)$.
\end{proof}

\section{Conclusion}\label{sec:conclusions}

In this paper, we studied the complexity of problems related to relational marginal polytopes. As our first main contribution, we proved that domain-liftability of computing the partition function of a Markov logic network carries over to the problem of constructing relational marginal polytopes, which allowed us to extend positive results on domain-liftability for weight learning of Markov logic networks from \cite{kuzelka.aistats.2019}. As our second main contribution, we showed the hardness of deciding whether a point is contained in a relational marginal polytope, assuming a widely believed complexity-theoretic conjecture.

In this paper, we were interested only in answering the theoretical questions of domain-liftability: what is and is not possible. The next step is to design algorithms that will also be efficient in practice.

\vspace{0.2cm}
\noindent {\bf Acknowledgements.} OK's work has been supported by the OP VVV project {\it CZ.02.1.01/0.0/0.0/16\_019/0000765} ``Research Center for Informatics'' and a donation from X-Order Lab. Part of this work was done while OK was already supported by the Czech Science Foundation project ``Generative Relational Models'' (20-19104Y).

% \bibliographystyle{unsrt}
% \bibliography{ref}

 \ifappendix
% \newpage
\appendix
\section{A Lemma Used in Theorem \ref{theorem:partition_function}}

% \subsection{Lemmas and Omitted Proofs}

\begin{lemma}\label{lemma:minimalset}
Given a set $S$ of linear inequality constraints, there is an algorithm to find a minimal subset $S'\subseteq S$ such that $S'$ specifies the same polytope as $S$, in polynomial time in the size of $S$. 
\end{lemma}
\begin{proof}
Without loss of generality, we assume that every constraint $c_j$ in $S$ is of the form $\sum_i a_{j,i} x_i \le b_j.$ We construct $|S|$ linear programs: The $i$-th linear program uses all constraints in $S$ except $c_j$ as the constraints, and its objective function is $\max \sum_i a_{j,i} x_i$. If the optimal solution of this linear program is strictly larger than $b_j$, then we add $c_j$ into $S'$. It is not difficult to see that every constraint in $S'$ cannot be implied by other constraints, or else that constraint cannot be added into $S'$, so $S'$ is minimal. Besides, we only have $|S|$ linear programs each of which can be solved in polynomial time (e.g., using some interior-point methods \cite{boyd2004convex}), hence the whole procedure is in polynomial time. 
\end{proof}

 \fi

\bibliographystyle{aaai}
\bibliography{ref}

\begin{thebibliography}{}

\bibitem[\protect\citeauthoryear{Beame \bgroup et al\mbox.\egroup
  }{2015}]{beame2015symmetric}
Beame, P.; Van~den Broeck, G.; Gribkoff, E.; and Suciu, D.
\newblock 2015.
\newblock Symmetric weighted first-order model counting.
\newblock In {\em Proceedings of the 34th ACM SIGMOD-SIGACT-SIGAI Symposium on
  Principles of Database Systems},  313--328.
\newblock ACM.

\bibitem[\protect\citeauthoryear{Boyd and Vandenberghe}{2004}]{boyd2004convex}
Boyd, S., and Vandenberghe, L.
\newblock 2004.
\newblock {\em Convex optimization}.
\newblock Cambridge university press.

\bibitem[\protect\citeauthoryear{Bresler, Gamarnik, and
  Shah}{2014}]{bresler2014hardness}
Bresler, G.; Gamarnik, D.; and Shah, D.
\newblock 2014.
\newblock Hardness of parameter estimation in graphical models.
\newblock In {\em Advances in Neural Information Processing Systems},
  1062--1070.

\bibitem[\protect\citeauthoryear{De~Salvo~Braz, Amir, and
  Roth}{2005}]{braz2005lifted}
De~Salvo~Braz, R.; Amir, E.; and Roth, D.
\newblock 2005.
\newblock Lifted first-order probabilistic inference.
\newblock In {\em Proceedings of the 19th international joint conference on
  Artificial intelligence},  1319--1325.
\newblock Citeseer.

\bibitem[\protect\citeauthoryear{Gogate and
  Domingos}{2011}]{gogate2011probabilistic}
Gogate, V., and Domingos, P.
\newblock 2011.
\newblock Probabilistic theorem proving.
\newblock In {\em Proceedings of the Twenty-Seventh Conference on Uncertainty
  in Artificial Intelligence},  256--265.
\newblock AUAI Press.

\bibitem[\protect\citeauthoryear{Jaeger and Schulte}{2018}]{jaegerschulte}
Jaeger, M., and Schulte, O.
\newblock 2018.
\newblock Inference, learning, and population size: Projectivity for {SRL}
  models.
\newblock {\em CoRR} abs/1807.00564.

\bibitem[\protect\citeauthoryear{Jaeger}{2015}]{jaeger}
Jaeger, M.
\newblock 2015.
\newblock Lower complexity bounds for lifted inference.
\newblock {\em {TPLP}} 15(2):246--263.

\bibitem[\protect\citeauthoryear{Jones and Selman}{1974}]{jones1974turing}
Jones, N.~D., and Selman, A.~L.
\newblock 1974.
\newblock Turing machines and the spectra of first-order formulas.
\newblock {\em The Journal of Symbolic Logic} 39(1):139--150.

\bibitem[\protect\citeauthoryear{Kazemi \bgroup et al\mbox.\egroup
  }{2016}]{DBLP:conf/nips/KazemiKBP16}
Kazemi, S.~M.; Kimmig, A.; den Broeck, G.~V.; and Poole, D.
\newblock 2016.
\newblock New liftable classes for first-order probabilistic inference.
\newblock In {\em Advances in Neural Information Processing Systems 29: Annual
  Conference on Neural Information Processing Systems 2016},  3117--3125.

\bibitem[\protect\citeauthoryear{Kersting, Mladenov, and
  Tokmakov}{2017}]{DBLP:journals/ai/KerstingMT17}
Kersting, K.; Mladenov, M.; and Tokmakov, P.
\newblock 2017.
\newblock Relational linear programming.
\newblock {\em Artif. Intell.} 244:188--216.

\bibitem[\protect\citeauthoryear{Ku\v{z}elka and
  Kungurtsev}{2019}]{kuzelka.aistats.2019}
Ku\v{z}elka, O., and Kungurtsev, V.
\newblock 2019.
\newblock Lifted weight learning of markov logic networks revisited.
\newblock In {\em Proceedings of the 22nd International Conference on
  Artificial Intelligence and Statistics ({AISTATS-19})}.

\bibitem[\protect\citeauthoryear{Ku\v{z}elka \bgroup et al\mbox.\egroup
  }{2018}]{kuzelka2018relational}
Ku\v{z}elka, O.; Wang, Y.; Davis, J.; and Schockaert, S.
\newblock 2018.
\newblock Relational marginal problems: Theory and estimation.
\newblock In {\em Proceedings of the Thirty-Second AAAI Conference on
  Artificial Intelligence (AAAI-18)}.

\bibitem[\protect\citeauthoryear{Mittal \bgroup et al\mbox.\egroup
  }{2019}]{mittal}
Mittal, H.; Bhardwaj, A.; Gogate, V.; and Singla, P.
\newblock 2019.
\newblock Domain-size aware markov logic networks.
\newblock In {\em The 22nd International Conference on Artificial Intelligence
  and Statistics, {AISTATS} 2019},  3216--3224.

\bibitem[\protect\citeauthoryear{Mladenov, Ahmadi, and
  Kersting}{2012}]{mladenov}
Mladenov, M.; Ahmadi, B.; and Kersting, K.
\newblock 2012.
\newblock Lifted linear programming.
\newblock In {\em Proceedings of the Fifteenth International Conference on
  Artificial Intelligence and Statistics, {AISTATS} 2012},  788--797.

\bibitem[\protect\citeauthoryear{Richardson and
  Domingos}{2006}]{Richardson2006}
Richardson, M., and Domingos, P.
\newblock 2006.
\newblock Markov logic networks.
\newblock {\em Machine Learning} 62(1-2):107--136.

\bibitem[\protect\citeauthoryear{Roughgarden and
  Kearns}{2013}]{roughgarden2013marginals}
Roughgarden, T., and Kearns, M.
\newblock 2013.
\newblock Marginals-to-models reducibility.
\newblock In {\em Advances in Neural Information Processing Systems},
  1043--1051.

\bibitem[\protect\citeauthoryear{Shalizi and
  Rinaldo}{2013}]{shalizi2013consistency}
Shalizi, C.~R., and Rinaldo, A.
\newblock 2013.
\newblock Consistency under sampling of exponential random graph models.
\newblock {\em Annals of statistics} 41(2):508.

\bibitem[\protect\citeauthoryear{Sontag and Jaakkola}{2008}]{sontag2008new}
Sontag, D., and Jaakkola, T.~S.
\newblock 2008.
\newblock New outer bounds on the marginal polytope.
\newblock In {\em Advances in Neural Information Processing Systems},
  1393--1400.

\bibitem[\protect\citeauthoryear{Van~den Broeck, Meert, and
  Darwiche}{2014}]{van2014skolemization}
Van~den Broeck, G.; Meert, W.; and Darwiche, A.
\newblock 2014.
\newblock Skolemization for weighted first-order model counting.
\newblock In {\em Fourteenth International Conference on the Principles of
  Knowledge Representation and Reasoning}.

\bibitem[\protect\citeauthoryear{Van~den Broeck}{2011}]{broeck2011completeness}
Van~den Broeck, G.
\newblock 2011.
\newblock On the completeness of first-order knowledge compilation for lifted
  probabilistic inference.
\newblock In {\em Advances in Neural Information Processing Systems},
  1386--1394.

\bibitem[\protect\citeauthoryear{{Van Haaren} \bgroup et al\mbox.\egroup
  }{2016}]{van.haaren.mlj}
{Van Haaren}, J.; {Van den Broeck}, G.; Meert, W.; and Davis, J.
\newblock 2016.
\newblock Lifted generative learning of markov logic networks.
\newblock {\em Machine Learning} 103(1):27--55.

\end{thebibliography}

\end{document}